\documentclass{article}

\pdfoutput=1

\usepackage[final]{nips_2017}

\usepackage[utf8]{inputenc} 
\usepackage[T1]{fontenc}    
\usepackage{hyperref}       
\usepackage{url}            
\usepackage{booktabs}       
\usepackage{amsfonts}       
\usepackage{nicefrac}       
\usepackage{microtype}      

\usepackage{graphicx} 

\usepackage[textfont=md,font=footnotesize]{caption}
\usepackage[compact]{titlesec}

\title{Inverse Reward Design}

\author{
  Dylan Hadfield-Menell \hspace{10pt} Smitha Milli \hspace{10pt} Pieter Abbeel\thanks{OpenAI, International Computer Science Institute (ICSI)} \hspace{10pt} Stuart Russell \hspace{10pt} Anca D. Dragan\\\
  Department of Electrical Engineering and Computer Science\\
  University of California, Berkeley\\
  Berkeley, CA 94709 \\
  \texttt{\{dhm, smilli, pabbeel, russell, anca\}@cs.berkeley.edu}
}

\usepackage{amssymb}
\usepackage{amsmath}
\usepackage{amsthm}

\usepackage{color}

\newcommand{\mnor}{\mprox}
\newcommand{\rtrue}{\ensuremath{r^*}}
\newcommand{\agentmodel}{\ensuremath{\pi(\cdot | \rprox, \mnor)}}
\newcommand{\wprox}{\ensuremath{\overset{\sim}{w}}}
\newcommand{\wtrue}{\ensuremath{{w^*}}}
\newcommand{\zprox}{\ensuremath{\overset{\sim}{Z}}}
\newcommand{\rprox}{\ensuremath{\overset{\sim}{r}}}
\newcommand{\rstar}{\ensuremath{r^*}}
\newcommand{\proxspace}{\ensuremath{\overset{\sim}{\mathcal{R}}}}
\newcommand{\rspace}{\ensuremath{\mathcal{R}}}
\newcommand{\mprox}{\ensuremath{\overset{\sim}{M}}}

\newcommand{\piprox}{\ensuremath{\pi(\xi|\wprox,\mprox)}}

\newcommand{\phiprox}{\ensuremath{\overset{\sim}{\phi}}}

\newcommand{\secref}[1]{Section~\ref{#1}}

\newcommand{\eqnref}[1]{Equation~\ref{#1}}

\newcommand{\figref}[1]{Figure~\ref{#1}}

\newtheorem{defn}{Definition}

\newtheorem{assume}{Assumption}
\newtheorem{prop}{Proposition}

{\begin{list}{$\bullet$}{%
    \setlength{\topsep}{0in}
    \setlength{\partopsep}{0in}
    \setlength{\itemsep}{0in}
    \setlength{\parsep}{0in}
    \setlength{\leftmargin}{3.5em}
    \setlength{\rightmargin}{0in}
    \setlength{\itemindent}{-.1in}
}
}%
{\end{list}
}

\DeclareMathOperator*{\argmax}{argmax}
\DeclareMathOperator*{\expect}{\mathbb{E}}

\newcommand{\prg}[1]{\noindent\textbf{#1. }}

\begin{document}

\maketitle

\begin{abstract}
Autonomous agents optimize the reward function we give them. What they don't know is how hard it is for us to design a reward function that actually captures what we  want. When designing the reward, we might think of some specific training scenarios, and make sure that the reward will lead to the right behavior in \emph{those} scenarios. Inevitably, agents encounter \emph{new} scenarios (e.g., new types of terrain) where optimizing that same reward may lead to undesired behavior.
Our insight is that reward functions are merely \emph{observations} about what the designer \emph{actually} wants, and that they should be interpreted in the context in which they were designed. We introduce \emph{inverse reward design} (IRD) as the problem of inferring the true objective based on the designed reward and the training MDP. We introduce approximate methods for solving IRD problems, and use their solution to plan risk-averse behavior in test MDPs. Empirical results suggest that this approach can help alleviate negative side effects of misspecified reward functions and mitigate reward hacking.
\end{abstract}

\section{Introduction}
\label{sec-intro}
Robots\footnote{Throughout this paper, we will use
 robot to refer generically to any artificial agent.} are becoming more capable of optimizing their reward functions. But along with that
comes the burden of making sure we specify these reward functions correctly. Unfortunately, this is a notoriously difficult task.
 Consider the example from \figref{fig-teaser}. Alice, an AI engineer, wants to build a robot, we'll call it Rob, for mobile
 navigation. She wants it to reliably navigate to a target location and
 expects it to primarily encounter grass lawns and dirt pathways. She
 trains a perception system to identify each of these terrain types and
 then uses this to define a reward function that incentivizes moving
 towards the target quickly, avoiding grass where possible. When Rob is deployed into the world, it encounters a novel terrain
 type; for dramatic effect, we'll suppose that it is lava. The terrain prediction goes haywire on this out-of-distribution input and generates a meaningless classification which, in turn, produces an arbitrary reward evaluation. As a result, Rob might then drive to its demise. This failure occurs because the
 reward function Alice \emph{specified} implicitly through the terrain predictors, which ends up outputting arbitrary values for lava, is \emph{different} from the  one Alice \emph{intended}, which would actually penalize traversing lava.

In the terminology from \cite{amodei2016concrete}, this is a \emph{negative side effect} of a \emph{misspecified reward} --- a failure mode of reward design where leaving out important aspects leads to poor behavior. Examples date back to King Midas, who wished that everything he touched turn to gold, leaving out that he didn't mean his food or family. Another failure mode is \emph{reward hacking}, which happens when, e.g., a vacuum cleaner ejects collected dust so that it can collect even more~\citep{Russell+Norvig:2010}, or a racing boat in a game loops in place to collect points instead of actually winning the race~\citep{amodei2016faulty}. 
Short of requiring that the reward designer anticipate and penalize
 all possible misbehavior in advance, how can we alleviate the impact of such
 reward misspecification? 
 
\emph{We leverage a key insight: that the designed reward function should merely be an observation about the intended reward, rather than the definition; and should be interpreted in the context in which it was designed. }
First,
a robot should have uncertainty about its reward function, instead of
treating it as fixed. This enables it to, e.g., be risk-averse
when planning in scenarios where it is not clear what the right answer is, or to ask for help. Being uncertain about the true reward, however, is
only half the battle. To be effective, a robot must acquire the
\emph{right} kind of uncertainty, i.e. know what it knows and what it doesn't. We propose that the `correct' shape of this
uncertainty depends on the environment for which the reward was
designed. 

In Alice's case, the situations where she tested Rob's
learning behavior did not contain lava. Thus, the lava-\emph{avoiding}
reward would have produced \emph{the same behavior} as Alice's designed reward function in the (lava-free)
environments that Alice considered. A robot that knows the settings it
was evaluated in should also know that, even though the designer
specified a lava-\emph{agnostic} reward, they might have actually
meant the lava-\emph{avoiding} reward. Two reward functions that would produce similar behavior in
the training environment should be treated as equally likely, regardless of which one the designer actually specified.
We formalize this in a probabilistic
model that relates the proxy (designed) reward to the true reward via the
following assumption:
\vspace{-.5em}
\begin{assume}
\label{assume-proxy-reward}
Proxy reward functions are likely to the extent that they lead to high \textbf{true}
utility behavior in the \textbf{training} environment. 
\end{assume}
\vspace{-.5em}

\begin{figure}
\centering
\includegraphics[width=.9\columnwidth]{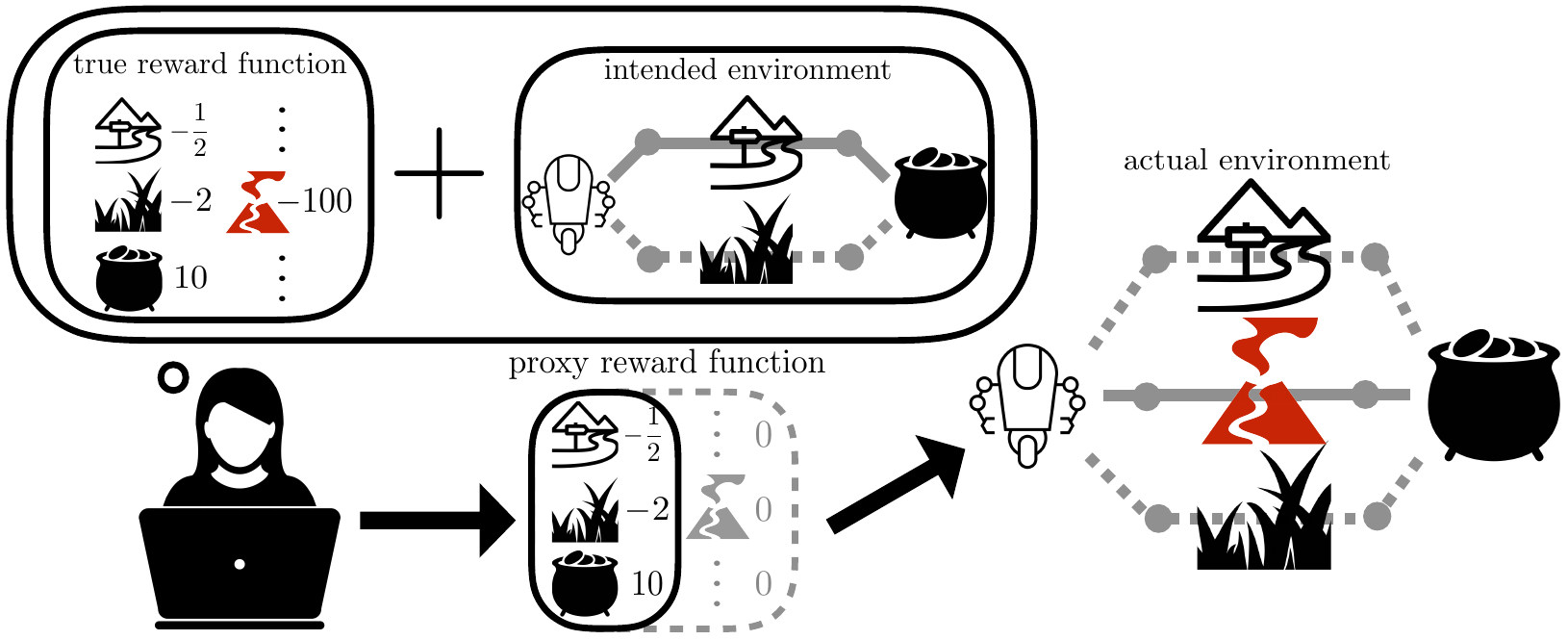}
\caption{{\small An illustration of a \emph{negative side effect}. Alice designs a reward function so that her robot navigates to the pot of gold and prefers dirt paths. She does not consider that her robot might encounter lava in the real world and leaves that out of her reward specification. The robot maximizing this \emph{proxy} reward function drives through the lava to its demise. In this work, we formalize the (Bayesian) \emph{inverse reward design} (IRD) problem as the problem of inferring (a distribution on) the \emph{true} reward function from the proxy. We show that IRD can help mitigate unintended consequences from misspecified reward functions like negative side effects and reward hacking.} } \label{fig-teaser}
\end{figure}

Formally, we assume that the observed proxy reward function is the approximate
solution to a \emph{reward design
  problem}~\citep{singh2010where}. Extracting the true reward is the
\emph{inverse} reward design problem. 

The idea of using human behavior as observations about the reward
function is far from new. Inverse reinforcement learning uses human
demonstrations~\citep{ng2000algorithms, ziebart2008maximum}, shared
autonomy uses human operator control
signals~\citep{javdani2015shared}, preference-based reward learning
uses answers to comparison queries~\citep{jain2015learning}, and even
what the human wants~\citep{HadfieldMenell2017offswitch}. We observe
that, \emph{even when the human behavior is to actually write down a
  reward function}, this should still be treated as an observation,
demanding its own observation model.

Our paper makes three contributions. First, we define the
\emph{inverse reward design} (IRD) problem as the problem of inferring
the true reward function given a proxy reward function, an intended
decision problem (e.g., an MDP), and a set of possible reward
functions. Second, we propose a solution to IRD and justify how an
intuitive algorithm which treats the proxy reward as a set of expert
demonstrations can serve as an effective approximation. Third, we show
that this inference approach, combined with risk-averse planning,
leads to algorithms that are robust to misspecified rewards, alleviating both negative side effects as well as reward hacking. We build
a system that `knows-what-it-knows' about reward evaluations that
automatically detects and avoids distributional shift in situations
with high-dimensional features. Our approach substantially outperforms
the baseline of literal reward interpretation.

\section{Inverse Reward Design}
\label{sec-ird}

 \begin{defn}(Markov Decision Process~\cite{puterman2009markov})
A (finite-horizon) \emph{Markov decision process} (MDP), $M$, is a
tuple $M = \langle \mathcal{S}, \mathcal{A}, T, r, H
\rangle$. $\mathcal{S}$ is a set of states. $\mathcal{A}$ is a set of
actions. $T$ is a probability distribution over the next state, given
the previous state and action. We write this as $T(s_{t+1}| s_t,
a)$. $r$ is a reward function that maps states to rewards
$r: \mathcal{S} \mapsto \mathbb{R}$. $H\in \mathbb{Z}_+$ is the finite planning
horizon for the agent.
\end{defn}

A solution to $M$ is a \emph{policy}: a mapping from the current
timestep and state to a distribution over actions. The optimal policy
maximizes the expected sum of rewards. We will use $\xi$ to represent trajectories. In this work, we consider reward functions that are linear combinations of feature
vectors $\phi(\xi).$ Thus, the reward for a trajectory, given weights $w$, is $r(\xi;~w)~=~w^\top\phi(\xi).$ 

The MDP formalism defines optimal behavior, given a reward
function. However, it provides no information about where this reward
function comes from~\citep{singh2010where}. We refer to an MDP without rewards as a \emph{world model}. In practice, a system designer needs to select a
reward function that encapsulates the intended behavior. This process
is \emph{reward engineering} or \emph{reward
design}:

\begin{defn} 
(Reward Design Problem~\citep{singh2010where}) A \emph{reward design problem} (RDP) is
  defined as a tuple $P = \langle \rstar, \mnor, \proxspace, \agentmodel\rangle$. \rstar{} is the true reward function. \mnor{} is a world model. \proxspace{} is
  a set of proxy reward functions. \agentmodel{} is an agent model, that defines a distribution on trajectories given a (proxy) reward function and a world model. 
\end{defn}

In an RDP, the designer believes that an agent, represented by the policy \agentmodel, will be deployed in \mnor. She must specify a \emph{proxy reward function} $\rprox
\in \proxspace$ for the agent. Her goal is to
specify \rprox{} so that \agentmodel{} obtains high
reward according to the true reward function \rstar. We let \wprox{} 
represent weights for the proxy reward function and \wtrue{}
represent weights for the true reward function.


In this work, our motivation is that system designers are fallible, so we should not expect that they perfectly solve the reward design problem. Instead we consider the case where the system designer is approximately optimal at solving a known RDP, which is distinct from the MDP that the robot currently finds itself in. By inverting the reward design process to infer (a distribution on) the
true reward function \rtrue, the robot can understand where its reward evaluations have high variance and plan to avoid those states. We refer to this inference problem as the \emph{inverse reward design} problem:

\begin{defn}
(Inverse Reward Design) The \emph{inverse reward design} (IRD) problem is defined by a tuple $\langle \rspace, \mnor,
  \proxspace, \agentmodel, \rprox \rangle$. \rspace{} is a space of
  possible reward functions. \mnor{} is a world model. $\langle -, \mnor, \proxspace, \agentmodel \rangle$  partially specifies an RDP $P$, with an unobserved reward function $\rtrue\in \rspace$. $\rprox \in \proxspace$ is the observed proxy reward that is an (approximate) solution to $P$.
\end{defn}
In solving an IRD problem, the goal is to recover \rtrue. We will explore Bayesian approaches to IRD, so we will assume a prior distribution on \rtrue{} and infer a posterior distribution on \rtrue{} given \rprox{} $P(\rtrue | \rprox, \mnor)$.

\section{Related Work}
\label{sec-related}

 \prg{Optimal reward design} \cite{singh2010where} formalize and
 study the problem of designing optimal rewards. They consider a
 designer faced with a distribution of environments, a class of reward
 functions to give to an agent, and a \emph{fitness} function. They
 observe that, in the case of bounded agents, it may be optimal to
 select a proxy reward that is distinct from the fitness
 function. \cite{sorg2010reward} and subsequent work has studied the
 computational problem of selecting an optimal proxy reward. 

In our work, we consider an alternative situation where
the \emph{system designer} is the bounded agent. In this case, the proxy
reward function is distinct from the fitness function -- the true
utility function in our terminology -- because system designers can make
mistakes. IRD formalizes the problem of determining a true utility
function given an observed proxy reward function. This enables us to
design agents that are robust to misspecifications in their reward
function.

\prg{Inverse reinforcement learning} In \emph{inverse reinforcement
 learning}
 (IRL) \citep{ng2000algorithms,ziebart2008maximum,evans2016learning,syed2007game}
 the agent observes demonstrations of (approximately) optimal behavior
 and infers the reward function being optimized. IRD is a similar
 problem, as both approaches infer an unobserved reward function. The difference is in the observation: IRL  observes behavior, while IRD directly observes a reward function. Key to IRD is assuming that this observed reward incentivizes 
 behavior that is approximately optimal with respect to the true reward. 
 In \secref{sec:ird-approximations}, we show how ideas from IRL can be used to approximate IRD. 
 Ultimately, we
 consider both IRD and IRL to be complementary strategies for \emph{value
 alignment}~\citep{cirl16}: approaches that allow designers or users to
 communicate preferences or goals.


 \prg{Pragmatics} The \emph{pragmatic} interpretation of language is
 the interpretation of a phrase or utterance in the context of
 alternatives~\citep{grice1975logic}. For
 example, the utterance ``some of the apples are red'' is often
 interpreted to mean that ``not all of the apples are red'' although
 this is not literally implied. This is because, in context, we typically
 assume that a speaker who meant to say ``all the apples are red''
 would simply say so.

 Recent models of pragmatic language interpretation use two levels of
 Bayesian reasoning~\citep{frank2009informative, goodman2014probabilistic}. At the lowest level, there is a
 literal listener that interprets language according to a shared
 literal definition of words or utterances. Then, a speaker selects
 words in order to convey a particular meaning to the literal
 listener. To model pragmatic inference, we consider the probable meaning
 of a given utterance from this speaker. We can think of IRD as a
 model of pragmatic reward interpretation: the speaker in pragmatic
 interpretation of language is directly analogous to the reward
 designer in IRD.

\section{Approximating the Inference over True Rewards}
\label{sec-generative}

We solve IRD problems by formalizing Assumption~\ref{assume-proxy-reward}: the idea that proxy reward functions are
 likely to the extent that they incentivize high utility behavior in the training MDP. This will give us a probabilistic model for how \wprox{} is generated from the true \wtrue{} and the training MDP \mprox{}. We will invert this probability model to compute a
 distribution $P(w=\wtrue|\wprox,\mprox)$ on the true utility
function.
 
 \subsection{Observation Model} Recall that \piprox{} is the designer's model of the probability that the robot will select trajectory $\xi$, given proxy reward \wprox. We will assume that \piprox{} is the maximum entropy trajectory
distribution from \cite{ziebart2008maximum}, i.e. the designer models the robot as approximately optimal: $\piprox~\propto~\exp(w^\top~\phi(\xi)).$  An optimal designer chooses \wprox{} to maximize expected true value, i.e. $\mathbb{E}[\wtrue^\top \phi(\xi) | \xi \sim \piprox]$ is high. We model an approximately optimal designer:



\begin{equation}
 P(\wprox | \wtrue,\mprox) \propto \exp\left(\beta \expect \left[\wtrue^\top \phi(\xi) | \xi \sim \piprox \right]\right) \label{eq-rdp}
\end{equation} with $\beta$ controlling how close to optimal we assume the person to be.
This is now a formal statement of Assumption~\ref{assume-proxy-reward}. \wtrue{} can be pulled out of the expectation, so we let 
$\phiprox~=~\expect[\phi(\xi)|\xi\sim\piprox].$ Our goal is to invert \eqref{eq-rdp} and sample from (or otherwise
estimate) $P(\wtrue|\wprox,\mprox)~\propto~P(\wprox|\wtrue,\mprox)P(\wtrue).$
The primary difficulty this entails is that we need to know the
\emph{normalized} probability $P(\wprox | \wtrue, \mprox).$ This depends on its normalizing constant, $\zprox(w)$, which integrates over possible proxy rewards.
\begin{equation}
 P(w=\wtrue | \wprox, \mprox) \propto \frac{\exp\left(\beta w^\top \phiprox \right)}{\zprox(w)}P(w), 
\zprox(w) = \int_{\wprox} \exp\left(\beta w^\top \phiprox \right) d\wprox. \label{eq-irdp}
\end{equation}


\subsection{Efficient approximations to the IRD posterior}
\label{sec:ird-approximations}

To compute $P(w=\wtrue | \wprox, \mprox)$, we must compute \zprox{}, which is intractable if $\wprox$ lies in an
infinite or large finite set. Notice that computing the value of the
integrand for \zprox{} is highly non-trivial as it involves solving a planning problem. This is an example of what is referred to as a
\emph{doubly-intractable} likelihood~\citep{murray2006mcmc}. We consider two methods to approximate this normalizing constant. 


\begin{figure*}
\centering
\includegraphics[width=\textwidth]{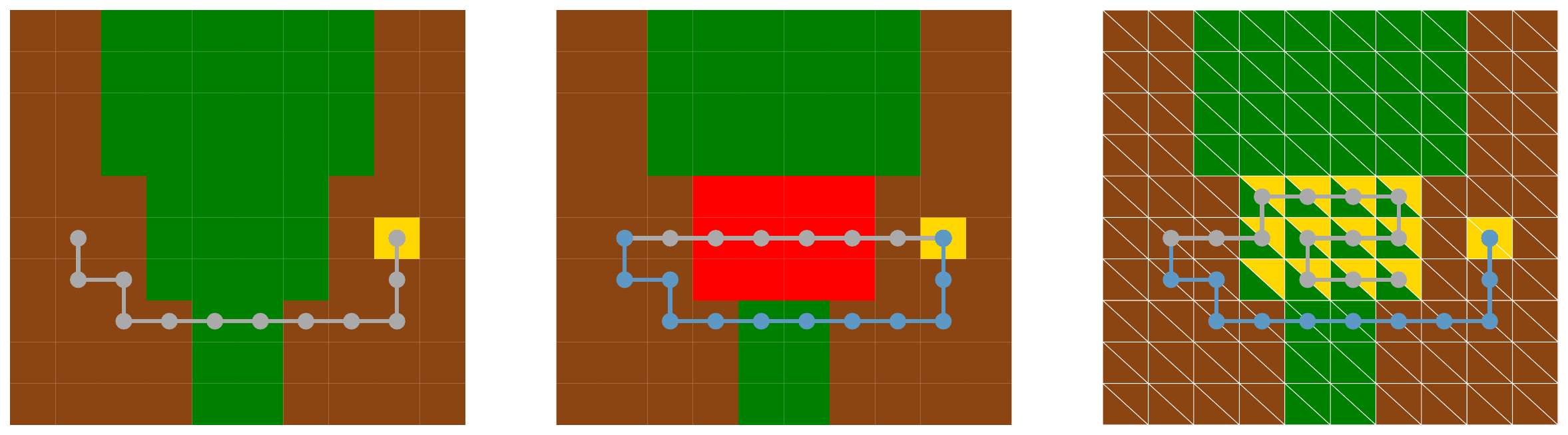}
\caption{{\small An example from the Lavaland domain. \textbf{Left:} The training MDP where the designer specifies a proxy reward function. This incentivizes movement toward targets (yellow) while preferring dirt (brown) to grass (green), and generates the gray trajectory. \textbf{Middle:} The testing MDP has lava (red). The proxy does not penalize lava, so optimizing it makes the agent go straight through (gray). This is a negative side effect, which the IRD agent avoids (blue): it treats the proxy as an observation in the context of the training MDP, which makes it realize that it cannot trust the (implicit) weight on lava. \textbf{Right:} The testing MDP has cells in which two sensor indicators no longer correlate: they look like grass to one sensor but target to the other. The proxy puts weight on the first, so the literal agent goes to these cells (gray). The IRD agent knows that it can't trust the distinction and goes to the target on which both sensors agree (blue). } } \label{fig-lavaland}
\end{figure*}

\prg{Sample to approximate the normalizing constant}
This approach, inspired by methods in approximate Bayesian
computation~\citep{sunnaaker2013approximate}, samples a finite set of weights $\{w_i\}$ to approximate the integral in \eqnref{eq-irdp}. We found empirically that it helped to include the candidate sample $w$ in the sum. This leads to the normalizing constant
\begin{equation}
\hat{Z}(w) = \exp\left(\beta w^\top \phi_w\right) + \sum_{i=0}^{N-1} \exp\left(\beta w^\top \phi_i\right).
\end{equation}
Where $\phi_i$ and $\phi_w$ are the vector of feature counts realized optimizing $w_i$ and $w$ respectively.

\prg{Bayesian inverse reinforcement learning} During inference, the normalizing constant serves a calibration purpose: it computes how good the behavior produced by all proxy rewards in that MDP would be with respect to the true reward. Reward functions which increase the reward for \emph{all} trajectories are not preferred in the inference. This creates an invariance to linear shifts in the feature encoding. If we were to change the MDP by shifting features by some vector $\phi_0$, $\phi \leftarrow \phi+\phi_0$, the posterior over $w$ would remain the same.

We can achieve a similar calibration and maintain the same property by directly integrating over the possible trajectories in the MDP:
\begin{equation}
Z(w) = \left(\int_\xi \exp(w^\top \phi(\xi)) d\xi \right)^\beta; \ \ \hat{P}(w | \wprox) \propto \frac{\exp\left( \beta w^\top
  \phiprox\right)}{Z(w)}\label{eq-irl}
\end{equation}

\begin{prop}
The posterior distribution that the IRD model induces on \wtrue{} (i.e., \eqnref{eq-irdp}) and the posterior distribution induced by IRL (i.e., \eqnref{eq-irl}) are invariant to linear translations of the features in the training MDP.
\end{prop}
\begin{proof}  First, we observe that this shift does not change the behavior of the planning agent due to linearity of the Bellman backup operation, i.e., $\phiprox' = \phiprox + \phi_0$. In \eqnref{eq-irdp} linearity of expectation allows us to pull $\phi_0$ out of the expectation to compute \phiprox:
\begin{align}
    \frac{\exp\left(\beta w^\top \phiprox' \right)}{\int_{\wprox} \exp\left(\beta w^\top \phiprox' \right)d\wprox} &= \frac{\exp\left(\beta w^\top \phi_0 \right)\exp\left(\beta w^\top \phiprox \right)}{\int_{\wprox} \exp\left(\beta w^\top \phi_0 \right)\exp\left(\beta w^\top \phiprox \right)d\wprox}\\
    &= \frac{\exp\left(\beta w^\top \phiprox \right)}{\int_{\wprox} \exp\left(\beta w^\top \phiprox \right)d\wprox}
\end{align}
This shows that \eqnref{eq-irdp} is invariant to constant shifts in the feature function. The same argument applies to \eqnref{eq-irl}.
\end{proof}





This choice of normalizing constant approximates the posterior to an IRD problem with the posterior from maximum entropy IRL~\citep{ziebart2008maximum}. The result has an intuitive interpretation. The proxy \wprox{} determines the average feature counts for a hypothetical dataset of expert demonstrations and $\beta$ determines the effective size of that dataset.  The agent solves \mprox{} with reward \wprox{} and computes the corresponding feature expectations \phiprox. The agent then pretends like it got $\beta$ demonstrations with features counts \phiprox, and runs IRL. The more the robot believes the human is good at reward design, the more demonstrations it pretends to have gotten from the person. The fact that reducing the proxy to behavior in \mprox{} approximates IRD is not surprising:  the main point of IRD is that the proxy \emph{reward} is merely a statement about what \emph{behavior} is good in the training environment.

\section{Evaluation}
\label{sec-experiments}

\subsection{Experimental Testbed}

We evaluated our approaches in a model of the scenario from \figref{fig-teaser} that we call \emph{Lavaland}. Our system designer, Alice, is programming a mobile robot, Rob. We model this as a gridworld with movement in the four cardinal directions and four terrain types: target, grass, dirt, and lava. The true objective for Rob, \wtrue, encodes that it should get to the target quickly, stay off the grass, and avoid lava. Alice designs a proxy that performs well in a training MDP that does not contain lava. Then, we measure Rob's performance in a test MDP that \emph{does} contain lava. Our results show that combining IRD and risk-averse planning creates incentives for Rob to avoid unforeseen scenarios.



We experiment with four variations of this environment: two proof-of-concept conditions in which the reward is misspecified, but the agent has direct access to feature indicators for the different categories (i.e. conveniently having a feature for lava); and two challenge conditions, in which the right features are \emph{latent}; the reward designer does not build an indicator for lava, but by reasoning in the raw observation space and then using risk-averse planning, the IRD agent still avoids lava.

\subsubsection{Proof-of-Concept Domains}
These domains contain feature indicators for the four categories: grass, dirt, target, and lava.

\prg{Side effects in Lavaland}
Alice expects Rob to encounter 3 types of terrain: grass,
dirt, and target, and so she only considers the training MDP from \figref{fig-lavaland} (left). She provides a \wprox{} to encode a
trade-off between path length and time spent on grass. 

The training MDP contains no lava, but it is introduced when Rob is deployed.
An agent that treats the proxy reward literally might go on the lava in the test MDP. However, an agent that runs IRD will know that it can't trust the weight on the lava indicator, since all such weights would produce the same behavior in the training MDP (\figref{fig-lavaland}, middle).
\begin{figure}
    \centering
    \includegraphics[width=.75\textwidth]{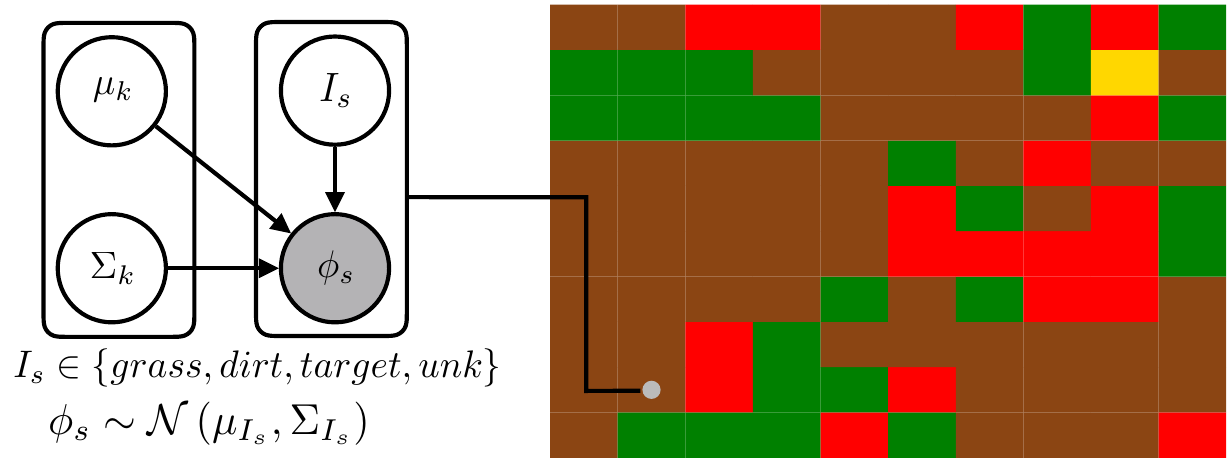}
    \caption{{\small Our challenge domain with latent rewards. Each terrain type (grass, dirt, target, lava) induces a different distribution over high-dimensional features: $\phi_{s} \sim \mathcal{N}(\mu_{I_{s}}, \Sigma_{I_{s}})$. The designer never builds an indicator for lava, and yet the agent still needs to avoid it in the test MDPs.} }
    \label{fig:hard-exps}
\end{figure}

\prg{Reward Hacking in Lavaland}
Reward hacking refers generally to reward functions that can be gamed
or tricked. 
To model this within Lavaland, we use features that are correlated
in the training domain but are uncorrelated in the testing
environment. There are 6 features: three from one sensor and three from another sensor. In the training environment the features from both sensors are correct indicators of the state's terrain category (grass, dirt, target).


At test time, this correlation gets broken: lava looks
like the target category to the second sensor, but the grass category to the first sensor. This is akin to how in a racing game~\citep{amodei2016faulty}, winning and game points can be correlated at reward design time, but test environments might contain loopholes for maximizing points without winning. We want agents to hedge their bets between winning and points, or, in Lavaland, between the two sensors. An agent that
treats the proxy reward function literally might go to these new cells if they are closer. In contrast, an agent that runs IRD will know that a reward function with the same weights put on the first sensor is just as likely as the proxy. Risk averse planning makes it go to the target for which both sensors agree (\figref{fig-lavaland}, right).

\subsubsection{Challenge Domain: Latent Rewards, No More Feature Indicators}
The previous examples allow us to
explore reward hacking and negative side effects in an isolated
experiment, but are unrealistic as they assume the existence of a feature indicator for unknown, unplanned-for terrain.
To investigate misspecified objectives in a more realistic
setting, we shift to the terrain type being latent, and inducing raw observations: we use a model where the terrain category determines the
mean and variance of a multivariate Gaussian distribution over observed features. \figref{fig:hard-exps} shows a depiction of this scenario. The designer has in mind a proxy reward on dirt, target, and grass, but \emph{forgets that lava might exist}. We consider two realistic ways through which a designer might actually specify the proxy reward function, which is based on the terrain types that the robot does not have access to: 1) directly on the \textbf{raw observations} --- collect samples of the training terrain types (dirt, grass, target) and train a (linear) reward predictor; or 2) \textbf{classifier features} --- build a classifier to classify terrain as dirt, grass, or target, and define a proxy on its output. 

Note that this domain allows for both negative side effects and reward hacking. Negative side effects can occur because the feature distribution for lava is different from the feature distribution for the three safe categories, and the proxy reward is trained only on the three safe categories. Thus in the testing MDP, the evaluation of the lava cells
will be arbitrary so maximizing the proxy reward will likely lead the
agent into lava. Reward hacking occurs when features that are correlated for the safe categories are uncorrelated for the lava category.

\subsection{Experiment}

\prg{Lavaland Parameters} We defined a distribution on map layouts with a log likelihood function that prefers maps where neighboring grid cells are the same. We mixed this log likelihood with a quadratic cost for deviating from a target ratio of grid cells to ensure similar levels of the lava feature in the testing MDPs. Our training MDP is 70\% dirt and 30\% grass. Our testing MDP is 5\% lava, 66.5\% dirt, and 28.5\% grass. 

In the proof-of-concept experiments, we selected the proxy reward function uniformly at random. For latent rewards, we picked a proxy reward function that evaluated to $+1$ for target, $+.1$ for dirt, and $-.2$ for grass. To define a proxy on raw observations, we sampled 1000 examples of grass, dirt, and target and did a linear regression. With classifier features, we simply used the target rewards as the weights on the classified features.  We used 50 dimensions for our feature vectors. We selected trajectories via \emph{risk-averse trajectory optimization}. Details of our planning method, and our approach and rationale in selecting it can be found in the supplementary material. 

\prg{IVs and DVs} We measured the fraction of runs that encountered a lava cell on the test MDP as our dependent measure. This tells us the proportion of trajectories where the robot gets 'tricked' by the misspecified reward function; if a grid cell has never been seen then a conservative robot should plan to avoid it. We manipulate two factors: \textbf{literal-optimizer} and \textbf{Z-approx}. \textbf{literal-optimizer} is true if the robot interprets the proxy reward literally and false otherwise. \textbf{Z-approx} varies the approximation technique used to compute the IRD posterior. It varies across the two levels described in \secref{sec:ird-approximations}: sample to approximate the normalizing constant (\textbf{Sample-Z}) or use the normalizing constant from maximum entropy IRL (\textbf{MaxEnt-Z})~\citep{ziebart2008maximum}. 

\begin{figure}
\centering
\includegraphics[width=.98\textwidth]{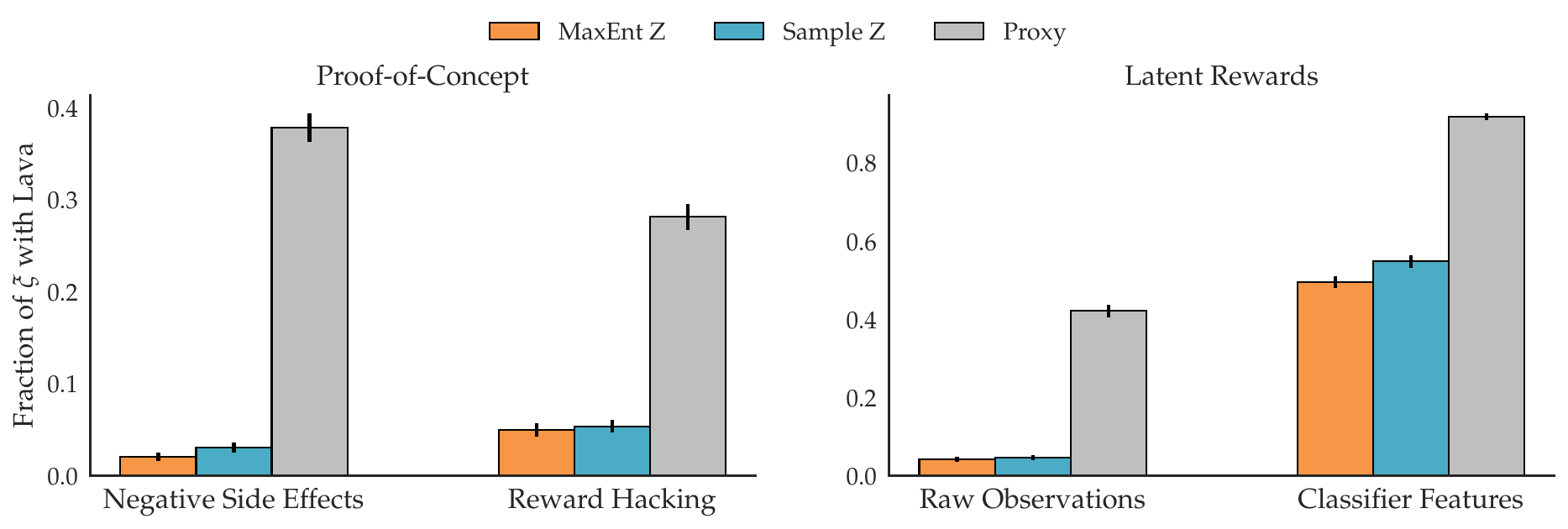}
\caption{{\small The results of our experiment comparing our proposed method to a baseline that directly plans with the proxy reward function. By solving an inverse reward design problem, we are able to create generic incentives to avoid unseen or novel states.} \vspace{-10pt}} \label{fig-results}
\end{figure}

\prg{Results}\figref{fig-results} compares the approaches. On the left, we see that IRD alleviates negative side effects (avoids the lava) and reward hacking (does not go as much on cells that look deceptively like the target to one of the sensors). This is important, in that the same inference method generalizes across different consequences of misspecified rewards. \figref{fig-lavaland} shows example behaviors.

In the more realistic latent reward setting, the IRD agent avoids the lava cells despite the designer forgetting to penalize it, and despite not even having an indicator for it: because lava is latent in the space, and so reward functions that would \emph{implicitly} penalize lava are as likely as the one actually specified, risk-averse planning avoids it.

We also see a distinction between {raw observations} and {classifier features}. The first essentially matches the proof-of-concept results (note the different axes scales), while the latter is much more difficult across all methods. The proxy performs worse because each grid cell is classified before being evaluated, so there is a relatively good chance that at least one of the lava cells is misclassified as target. IRD performs worse because the behaviors considered in inference plan in the already  classified terrain: a non-linear transformation of the features. The inference must both determine a good linear reward function to match the behavior \emph{and} discover the corresponding uncertainty about it. When the proxy is a linear function of raw observations, the first job is considerably easier.

\section{Discussion}
\label{sec-conclusion}

\noindent\textbf{Summary.} In this work, we motivated and introduced the \emph{Inverse Reward Design} problem as an approach to mitigate the risk from misspecified objectives. We introduced an observation model, identified the challenging inference problem this entails, and gave several simple approximation schemes. Finally, we showed how to use the solution to an inverse reward design problem to avoid side effects and reward hacking in a 2D navigation problem. We showed that we are able to avoid these issues reliably in simple problems where features are binary indicators of terrain type. Although this result is encouraging, in real problems we won't have convenient access to binary indicators for what matters. Thus, our challenge evaluation domain gave the robot access to only a high-dimensional observation space. The reward designer specified a reward based on this observation space which forgets to penalize a rare but catastrophic terrain. IRD inference still enabled the robot to  understand that rewards which would implicitly penalize the catastrophic terrain are also likely.

\noindent\textbf{Limitations and future work.} IRD gives the robot a posterior distribution over reward functions, but much work remains in understanding how to best leverage this posterior. Risk-averse planning can work sometimes, but it has the limitation that the robot does not just avoid bad things like lava, it also avoids potentially good things, like a giant pot of gold. We anticipate that leveraging the IRD posterior for follow-up queries to the reward designer will be key to addressing misspecified objectives.

Another limitation stems from the complexity of the environments and reward functions considered here. The approaches we used in this work rely on explicitly solving a planning problem, and this is a bottleneck during inference. In future work, we plan to explore the use of different agent models that plan approximately or leverage, e.g., meta-learning~\citep{duanRL2017} to scale IRD up to complex environments. Another key limitation is the use of linear reward functions. We cannot expect IRD to perform well unless the prior places weights on (a reasonable approximation to) the true reward function. If, e.g., we encoded terrain types as RGB values in Lavaland, there is unlikely to be a reward function in our hypothesis space that represents the true reward well. 

Finally, this work considers one relatively simple error model for the designer. This encodes some implicit assumptions about the nature and likelihood of errors (e.g., IID errors). In future work, we plan to investigate more sophisticated error models that allow for systematic biased errors from the designer and perform human subject studies to empirically evaluate these models. 

Overall, we are excited about the implications IRD has not only in the short term, but also about its contribution to the general study of the value alignment problem.

\section*{Acknowledgements}
This work was supported by the Center for Human Compatible AI and the Open Philanthropy Project, the Future of Life Institute, AFOSR, and NSF Graduate Research Fellowship Grant No. DGE 1106400.





\bibliography{biblio}

\begin{thebibliography}{24}
\providecommand{\natexlab}[1]{#1}
\providecommand{\url}[1]{\texttt{#1}}
\expandafter\ifx\csname urlstyle\endcsname\relax
  \providecommand{\doi}[1]{doi: #1}\else
  \providecommand{\doi}{doi: \begingroup \urlstyle{rm}\Url}\fi

\bibitem[Amodei \& Clark(2016)Amodei and Clark]{amodei2016faulty}
Amodei, Dario and Clark, Jack.
\newblock {Faulty Reward Functions in the Wild}.
\newblock \url{https://blog.openai.com/faulty-reward-functions/}, 2016.

\bibitem[Amodei et~al.(2016)Amodei, Olah, Steinhardt, Christiano, Schulman, and
  Man{\'{e}}]{amodei2016concrete}
Amodei, Dario, Olah, Chris, Steinhardt, Jacob, Christiano, Paul, Schulman,
  John, and Man{\'{e}}, Dan.
\newblock {Concrete Problems in AI Safety}.
\newblock \emph{CoRR}, abs/1606.06565, 2016.
\newblock URL \url{http://arxiv.org/abs/1606.06565}.

\bibitem[Duan et~al.(2016)Duan, Schulman, Chen, Bartlett, Sutskever, and
  Abbeel]{duanRL2017}
Duan, Yan, Schulman, John, Chen, Xi, Bartlett, Peter~L., Sutskever, Ilya, and
  Abbeel, Pieter.
\newblock {RL}$^2$: {Fast Reinforcement Learning via Slow Reinforcement
  Learning}.
\newblock \emph{CoRR}, abs/1611.02779, 2016.
\newblock URL \url{http://arxiv.org/abs/1611.02779}.

\bibitem[Evans et~al.(2016)Evans, Stuhlm{\"u}ller, and
  Goodman]{evans2016learning}
Evans, Owain, Stuhlm{\"u}ller, Andreas, and Goodman, Noah~D.
\newblock {Learning the Preferences of Ignorant, Inconsistent Agents}.
\newblock In \emph{Proceedings of the Thirtieth AAAI Conference on Artificial
  Intelligence}, pp.\  323--329. AAAI Press, 2016.

\bibitem[Frank et~al.(2009)Frank, Goodman, Lai, and
  Tenenbaum]{frank2009informative}
Frank, Michael~C, Goodman, Noah~D, Lai, Peter, and Tenenbaum, Joshua~B.
\newblock {Informative Communication in Word Production and Word Learning}.
\newblock In \emph{Proceedings of the 31st Annual Conference of the Cognitive
  Science Society}, pp.\  1228--1233. Cognitive Science Society Austin, TX,
  2009.

\bibitem[Goodman \& Lassiter(2014)Goodman and
  Lassiter]{goodman2014probabilistic}
Goodman, Noah~D and Lassiter, Daniel.
\newblock {Probabilistic Semantics and Pragmatics: Uncertainty in Language and
  Thought}.
\newblock \emph{Handbook of Contemporary Semantic Theory. Wiley-Blackwell}, 2,
  2014.

\bibitem[Grice(1975)]{grice1975logic}
Grice, H.~Paul.
\newblock \emph{Logic and Conversation}, pp.\  43--58.
\newblock Academic Press, 1975.

\bibitem[Hadfield-Menell et~al.(2016)Hadfield-Menell, Dragan, Abbeel, and
  Russell]{cirl16}
Hadfield-Menell, Dylan, Dragan, Anca, Abbeel, Pieter, and Russell, Stuart.
\newblock {Cooperative Inverse Reinforcement Learning}.
\newblock In \emph{Proceedings of the Thirtieth Annual Conference on Neural
  Information Processing Systems}, 2016.

\bibitem[Hadfield{-}Menell et~al.(2017)Hadfield{-}Menell, Dragan, Abbeel, and
  Russell]{HadfieldMenell2017offswitch}
Hadfield{-}Menell, Dylan, Dragan, Anca~D., Abbeel, Pieter, and Russell,
  Stuart~J.
\newblock {The Off-Switch Game}.
\newblock In \emph{Proceedings of the International Joint Conference on
  Artificial Intelligence}, 2017.

\bibitem[Jain et~al.(2015)Jain, Sharma, Joachims, and Saxena]{jain2015learning}
Jain, Ashesh, Sharma, Shikhar, Joachims, Thorsten, and Saxena, Ashutosh.
\newblock {Learning Preferences for Manipulation Tasks from Online Coactive
  Feedback}.
\newblock \emph{The International Journal of Robotics Research}, 34\penalty0
  (10):\penalty0 1296--1313, 2015.

\bibitem[Javdani et~al.(2015)Javdani, Bagnell, and
  Srinivasa]{javdani2015shared}
Javdani, Shervin, Bagnell, J.~Andrew, and Srinivasa, Siddhartha~S.
\newblock {Shared Autonomy via Hindsight Optimization}.
\newblock In \emph{Proceedings of Robotics: Science and Systems XI}, 2015.
\newblock URL \url{http://arxiv.org/abs/1503.07619}.

\bibitem[Markowitz(1968)]{markowitz1968portfolio}
Markowitz, Harry~M.
\newblock \emph{Portfolio selection: efficient diversification of investments},
  volume~16.
\newblock Yale university press, 1968.

\bibitem[Murray et~al.(2006)Murray, Ghahramani, and MacKay]{murray2006mcmc}
Murray, Iain, Ghahramani, Zoubin, and MacKay, David.
\newblock {MCMC for Doubly-Intractable Distributions}.
\newblock In \emph{Proceedings of the Twenty-Second Conference on Uncertainty
  in Artificial Intelligence}, 2006.

\bibitem[Ng \& Russell(2000)Ng and Russell]{ng2000algorithms}
Ng, Andrew~Y and Russell, Stuart~J.
\newblock {Algorithms for Inverse Reinforcement Learning}.
\newblock In \emph{Proceedings of the Seventeenth International Conference on
  Machine Learning}, pp.\  663--670, 2000.

\bibitem[Puterman(2009)]{puterman2009markov}
Puterman, Martin~L.
\newblock \emph{Markov Decision Processes: Discrete Stochastic Dynamic
  Programming}.
\newblock John Wiley \& Sons, 2009.

\bibitem[Rockafellar \& Uryasev(2000)Rockafellar and
  Uryasev]{rockafellar2000optimization}
Rockafellar, R~Tyrrell and Uryasev, Stanislav.
\newblock Optimization of conditional value-at-risk.
\newblock \emph{Journal of risk}, 2:\penalty0 21--42, 2000.

\bibitem[Russell \& Norvig(2010)Russell and Norvig]{Russell+Norvig:2010}
Russell, Stuart and Norvig, Peter.
\newblock \emph{Artificial Intelligence: {A} Modern Approach}.
\newblock Pearson, 2010.

\bibitem[Singh et~al.(2010)Singh, Lewis, , and Barto]{singh2010where}
Singh, Satinder, Lewis, {Richard L.}, , and Barto, {Andrew G.}
\newblock {Where do rewards come from?}
\newblock In \emph{Proceedings of the International Symposium on AI Inspired
  Biology - A Symposium at the AISB 2010 Convention}, pp.\  111--116, 2010.
\newblock ISBN 1902956923.

\bibitem[Sorg et~al.(2010)Sorg, Lewis, and Singh]{sorg2010reward}
Sorg, Jonathan, Lewis, Richard~L, and Singh, Satinder~P.
\newblock {Reward Design via Online Gradient Ascent}.
\newblock In \emph{Proceedings of the Twenty-Third Conference on Neural
  Information Processing Systems}, pp.\  2190--2198, 2010.

\bibitem[Sunn{\aa}ker et~al.(2013)Sunn{\aa}ker, Busetto, Numminen, Corander,
  Foll, and Dessimoz]{sunnaaker2013approximate}
Sunn{\aa}ker, Mikael, Busetto, Alberto~Giovanni, Numminen, Elina, Corander,
  Jukka, Foll, Matthieu, and Dessimoz, Christophe.
\newblock {Approximate {B}ayesian {C}omputation}.
\newblock \emph{PLoS Comput Biol}, 9\penalty0 (1):\penalty0 e1002803, 2013.

\bibitem[Syed \& Schapire(2007)Syed and Schapire]{syed2007game}
Syed, Umar and Schapire, Robert~E.
\newblock {A Game-Theoretic Approach to Apprenticeship Learning}.
\newblock In \emph{Proceedings of the Twentieth Conference on Neural
  Information Processing Systems}, pp.\  1449--1456, 2007.

\bibitem[Syed et~al.(2008)Syed, Bowling, and Schapire]{syed2008apprenticeship}
Syed, Umar, Bowling, Michael, and Schapire, Robert~E.
\newblock Apprenticeship learning using linear programming.
\newblock In \emph{Proceedings of the 25th International Conference on Machine
  Learning}, pp.\  1032--1039. ACM, 2008.

\bibitem[Tamar et~al.(2015)Tamar, Chow, Ghavamzadeh, and
  Mannor]{tamar2015policy}
Tamar, Aviv, Chow, Yinlam, Ghavamzadeh, Mohammad, and Mannor, Shie.
\newblock Policy gradient for coherent risk measures.
\newblock In \emph{Advances in Neural Information Processing Systems}, pp.\
  1468--1476, 2015.

\bibitem[Ziebart et~al.(2008)Ziebart, Maas, Bagnell, and
  Dey]{ziebart2008maximum}
Ziebart, Brian~D, Maas, Andrew~L, Bagnell, J~Andrew, and Dey, Anind~K.
\newblock {Maximum Entropy Inverse Reinforcement Learning}.
\newblock In \emph{Proceedings of the Twenty-Third AAAI Conference on
  Artificial Intelligence}, pp.\  1433--1438, 2008.

\end{thebibliography}
\bibliographystyle{icml2017}

\newpage

\section*{Appendix: Risk Averse Trajectory Optimization}
\label{sec-risk-averse}

Our overall strategy is to implement a system that `knows-what-it-knows' about reward evaluations. So far we have considered the problem of computing the robot's uncertainty about reward evaluations. We have not considered the problem of using that reward uncertainty. Here we describe several approaches and highlight an important nuance of planning under a distribution over utility evaluations. 

The most straightforward option is to maximize expected reward under this distribution. However, this will ignore the reward uncertainty we have worked so hard to achieve: planning in expectation under a distribution over utility is equivalent to planning with the mean of that distribution. Instead we will plan in a risk averse fashion that penalizes trajectories which have high variance over their utility. Of course, risk averse planning  is a rich field with a variety of approaches~\citep{markowitz1968portfolio, rockafellar2000optimization, tamar2015policy}. In future work, we intended to explore a variety of risk averse planning methods and evaluate their relative pros and cons for our application.

In this work, we will take a simple approach: given a set of weights $\{w_i\}$ sampled from our posterior $P(w|\wprox, \mprox)$, we will have the agent compute a trajectory that maximizes
reward under the \emph{worst case} $w_i$ in our set. We can do this in one of two ways.

\prg{Trajectory-wide reward} Given a set of weights
$\{w_i\}$ sampled from our posterior $P(w|\wprox)$, we will have the agent compute a trajectory that maximizes
reward under the \emph{worst case} $w_i$ in our set:
\begin{equation}
\xi^* = \argmax_\xi \min_{w\in\{w_i\}} w^\top \phi(\xi). \label{eq-min-full}
\end{equation}
This planning problem is no longer isomorphic to an MDP, as the reward
may not decompose per state. Trajectory optimization in this case can
be done via the linear programming approach described
in \cite{syed2008apprenticeship}. 

\prg{Time-step independent reward} An alternative is to take minimum
over weights on a per state basis:
\begin{equation}
\xi^* = \argmax_\xi \sum_{s_t\in\xi} \min_{w\in\{w_i\}} w^\top \phi(s_t). \label{eq-min-indep}
\end{equation}
This is more conservative, because it allows the minimizer to pick a different reward for each time step.

Directly applying this approach, however, may lead to poor
results. The reason is that, unlike maximizing expected reward, this
planning approach will be sensitive to the particular feature encoding
used. In maximizing expected reward, shifting all feature by a
constant vector $\phi_0$ will not change the optimal trajectory. \emph{The
same is no longer true for a risk averse approach.} 

For example, consider a choice between actions $a_1$ and $a_2$, with
features $\phi_1$ and $\phi_2$ respectively. If we shift the features
by a constant value $-\phi_2$ (i.e., set the feature values for the
second action to 0), then, unless $a_1$ is preferred to $a_2$
for \emph{every} weight in the posterior, the agent will always select
the second action. The zero values of feature encodings are typically
arbitrary, so this is clearly undesireable behavior.

Intuitively, this is because rewards are not absolute, they are relative. Rewards need a reference point. We thus need to compare the reward $w^\top\phi(\xi)$ to \emph{something}: to the reward of some reference features $c_i$.
We will study three approaches: comparing reward to the initial state, to the training feature counts, and to the expected reward across any trajectory.

\prg{Comparing to initial state}
One straightforward approach is to take a particular state or
trajectory and enforce that it has the same evaluation across each
$w_i$. For example, we can enforce that the features for the initial
state state is the 0 vector. This has the desirable property that the
agent will remain in place (or try to) when there is very high
variance in the reward estimates (i.e., the solution to IRD gives
little information about the current optimal trajectory). 

\prg{Comparing to training feature counts} An
third option is to use the expected features \phiprox{} as the
feature offset. In the case, the agent will default to trying to match
the features that it would have observed maximizing \wprox{}
in \mprox.

\prg{Comparing to other trajectories} An alternative is to define $c_i$ as the log of the normalizing
constant for the maximum entropy trajectory distribution:
\begin{equation}
c_i = \log \int_\xi \exp(w_i^\top\phi(\xi))d\xi.
\end{equation}
With this choice of $c_i$, we have that $w_i^\top \phi(\xi) - c_i
= \log P(\xi | w_i)$.  Thus, this approach will select trajectories
that compare relatively well to the options under all $w_i$. Loosely
speaking, we can think of it as controlling for the total amount of
reward available in the MDP.

\begin{figure}
\centering
\includegraphics[width=.98\columnwidth]{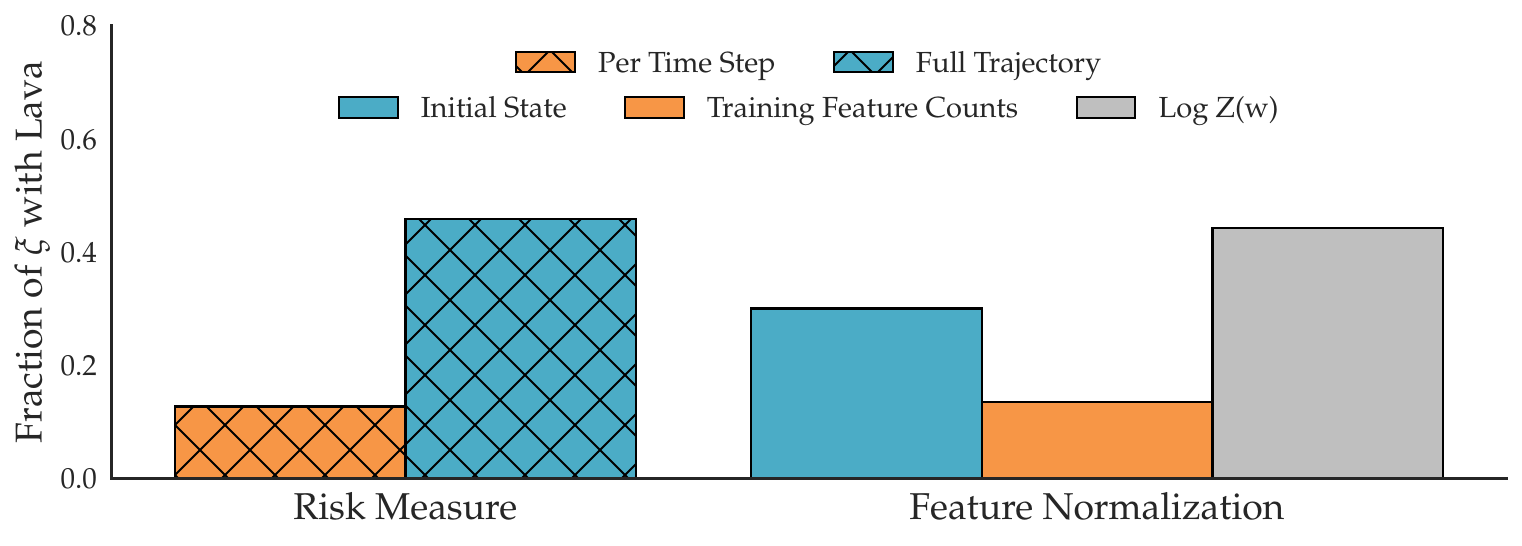}
\caption{{\small \textbf{Left:} We avoid side effects and reward hacking by computing a posterior distribution over reward function and then find a trajectory that performs well under the worst case reward function. This illustrates the impact of selecting this worst case independently per time step or once for the entire trajectory. Taking the minimum per time step increases robustness to the approximate inference algorithms used because we only need one particle in our sample posterior to capture the worst case for each grid cell type. For the full trajectory, we need a single particle to have inferred a worst case for \emph{every} grid cell type at once. \textbf{Right:} The impact of changing the offsets $c_i$. ``Initial State'' fixes the value of the start state to be 0. ``Training Feature Counts'' sets an average feature value from the training MDP to be 0. ``Log Z(w)'' offsets each evaluation by the normalizing from the maximum entropy trajectory distribution. This means that the sum of rewards across a trajectory is the log probability of a trajectory. }} \label{fig-granularity}
\end{figure}

\prg{Evaluation} Before running the full experiment, we did an initial internal comparison to find the best-performing planning method. We did a full factorial across the factors with the side effect feature encoding and the reward hacking feature encoding.

We found that the biggest overall change came from the \textbf{min-granularity} feature. A bar plot is shown in \figref{fig-granularity} (Left). Independently minimizing per time step was substantially more robust. We hypothesize that this is a downstream effect of the approximate inference used. We sample from our belief to obtain a particle representation of the posterior. Independently minimizing means that we need a single particle to capture the worst case for each grid cell type. Performing this minimization across the full trajectory means that a single particle has to faithfully represent the worst case for \emph{every} grid cell type.  

We also saw substantial differences with respect to the \textbf{reward-baseline} factor. \figref{fig-granularity} (Right) shows a bar plot of the results. In this case, setting the common comparison point to be the average feature counts from the training MDP performed best. We believe this is because of the similarity between the train and test scenarios: although there is a new grid cell present, it is still usually possible to find a trajectory that is similar to those available in the training MDP. We hypothesize that correctly making this decision will depend on the situation and we looking forward to exploring this in future work.

\end{document}